\documentclass[11pt]{article} \usepackage{rldm,palatino}
\usepackage{graphicx}

\author{
    Dustin Morrill\\
    Sony AI;\\
    Department of Computing Science\\
    University of Alberta; Amii\\
    Edmonton, Alberta, Canada\\
    \texttt{morrill@ualberta.ca}\\
    \And
    Esra'a Saleh\\
    Department of Computing Science\\
    University of Alberta; Amii\\
    Edmonton, Alberta, Canada\\
    \texttt{esraa1@ualberta.ca}\\
    \AND
    Michael Bowling\\
    Department of Computing Science\\
    University of Alberta; Amii\\
    Edmonton, Alberta, Canada\\
    \texttt{mbowling@ualberta.ca}\\
    \And
    Amy Greenwald\\
    Department of Computer Science\\
    Brown University\\
    Providence, Rhode Island, United States\\
    \texttt{amy\_greenwald@brown.edu}
}

\title{Interpolating Between Softmax Policy Gradient and Neural Replicator Dynamics with Capped Implicit Exploration}

\usepackage{mathtools}

\DeclarePairedDelimiter{\abs}{\lvert}{\rvert}

\DeclarePairedDelimiter{\subex}{(}{)}
\DeclarePairedDelimiter{\subblock}{[}{]}
\DeclarePairedDelimiter{\tuple}{(}{)}
\DeclarePairedDelimiter{\set}{\{}{\}}

\usepackage{stmaryrd}
\DeclarePairedDelimiter{\stopGrad}{\llbracket}{\rrbracket}

\usepackage{bbold}
\usepackage{bm}

\newcommand{\reals}{\mathbb{R}}

\newcommand{\Simplex}{\Delta}
\newcommand{\simplex}{\Simplex}

\newcommand{\bs}[1]{\bm{#1}}

\newcommand{\expectation}{\mathbb{E}}
\newcommand{\E}{\expectation}
\newcommand{\probability}{\mathbb{P}}
\newcommand{\Prob}{\probability}

\newcommand{\ones}{\bs{1}}

\newcommand{\bigO}[1]{\operatorname{\mathcal{O}}\subex{#1}}
\newcommand{\smallo}[1]{\operatorname{o}\subex{#1}}

\newcommand{\ip}[2]{\langle #1, \, #2 \rangle}
\newcommand{\ind}[1]{\mathbb{1}\set*{#1}}
\newcommand{\given}{\,|\,}

\newcommand{\where}{\;|\;}

\newcommand{\PureStratSet}{\mathcal{X}}
\newcommand{\PureStrategySet}{\PureStratSet}
\newcommand{\pureStrat}{x}
\newcommand{\PureStrat}{X}
\newcommand{\utility}{\upsilon}

\newcommand{\strategy}{\policy}
\newcommand{\strat}{\strategy}

\newcommand{\Actions}{\mathcal{A}}

\newcommand{\regret}{\rho}

\newcommand{\maxReward}{\supReward}
\newcommand{\grad}{\nabla}

\newcommand{\stepSize}{\alpha}

\newcommand{\infoSet}{I}

\newcommand{\Histories}{\mathcal{H}}
\newcommand{\TerminalHistories}{\mathcal{Z}}

\usepackage{amssymb}
\newcommand{\emptyHistory}{\varnothing}

\usepackage{upgreek}

\newcommand{\odpDecision}{\theta}

\newcommand{\est}{\widehat}

\newcommand{\policy}{\pi}
\newcommand{\reward}{r}
\newcommand{\PolicySet}{\Pi}

\newcommand{\stateValueFn}{q}

\newcommand{\RewardFn}{\reward}

\newcommand{\ObservationFn}{\omega}

\newcommand{\ObservationSet}{\mathcal{O}}
\newcommand{\DaimonStratSet}{\Sigma}
\newcommand{\DaimonStrategySet}{\DaimonStratSet}
\newcommand{\RandomReturn}{G}

\newcommand{\AgentStateSet}{\mathcal{S}}

\newcommand{\agentState}{s}

\newcommand{\daimonStrat}{\sigma}

\newcommand{\updateFn}{u}

\newcommand{\belief}{\xi}

\newcommand{\CixCorrection}{\beta}

 \renewcommand{\maxReward}{R^{\textsc{max}}}
\renewcommand{\stateValueFn}{v}
\newcommand{\actionValueFn}{q}
\newcommand{\ixParam}{\eta}
\newcommand{\maxReturn}{G^{\textsc{max}}}
 
\usepackage{xspace}

\makeatletter
\DeclareRobustCommand\onedot{\futurelet\@let@token\@onedot}
\def\@onedot{\ifx\@let@token.\else.\null\fi\xspace}

\def\eg/{\emph{e.g}\onedot} \def\Eg/{\emph{E.g}\onedot}
\def\ie/{\emph{i.e}\onedot} \def\Ie/{\emph{I.e}\onedot}
\def\cf/{\emph{c.f}\onedot} \def\Cf/{\emph{C.f}\onedot}
\def\vs/{\emph{vs}\onedot} \def\Vs/{\emph{Vs}\onedot}
\def\etc/{\emph{etc}\onedot}
\def\wrt/{with respect to} \def\dof/{d.o.f\onedot}
\def\etal/{\emph{et al}\onedot}
\def\viceversa/{\emph{vice-versa}}
\def\ow/{\emph{o.w}\onedot}
\def\whp/{w.h.p\onedot}
\def\apriori/{\emph{a priori}} \def\Apriori/{\emph{A priori}}
\def\ala/{\`{a} la}

\def\naive/{na\"{\i}ve} \def\Naive/{Na\"{\i}ve}
\def\rmPlus/{regret matching\textsuperscript{+}}
\def\rrmPlus/{RRM\textsuperscript{+}}
\def\rcfrPlus/{RCFR\textsuperscript{+}}
\def\cfrPlus/{CFR\textsuperscript{+}}
\@ifdefinable{\Politex/}{\def\Politex/{\textsc{Politex}}}

\def\NashConv/{\textsc{NashConv}}
\def\NashConvAUC/{$\overline{\textsc{NashConv}}$}

\def\heads/{\textsc{heads}}
\def\tails/{\textsc{tails}}
\def\even/{\textsc{even}}
\def\odd/{\textsc{odd}}

\makeatother

\def\docName/{thesis}
\def\DocName/{Thesis}
\def\rampName/{ramp}
\def\RampName/{Ramp}
\def\waitObservation/{\textsc{wait}}
 \makeatletter
\newcommand\safeIncCounter[1]{\@ifundefined{c@#1}{\newcounter{#1}\stepcounter{#1}}{\stepcounter{#1}}}
\makeatother

\usepackage{xargs}

\usepackage[capitalize]{cleveref}
\usepackage{xcolor}
\definecolor{offWhite}{RGB}{240,240,240}
\definecolor{grey}{RGB}{180,180,180}
\definecolor{darkgreen}{RGB}{0,125,0}
\definecolor{lime}{RGB}{255,200,0}

\definecolor{amiiBlue}{RGB}{16,72,118}
\definecolor{amiiPink}{RGB}{241,97,119}
\definecolor{amiiYellow}{RGB}{248,209,109}
\definecolor{amiiPurple}{RGB}{123,105,145}

\makeatletter
\ifdefined\algorithmic
  
\else
  \usepackage{algorithm}
\usepackage[noend]{algpseudocode}
\fi
\makeatother

\usepackage{amsmath, amssymb, amsfonts, amsthm}
\makeatletter
\ifdefined\theorem
\else
  \newtheorem{theorem}{Theorem}
\fi
\ifdefined\lemma
\else
  \newtheorem{lemma}{Lemma}
\fi
\ifdefined\corollary
\else
  \newtheorem{corollary}{Corollary}
\fi
\ifdefined\definition
\else
  \newtheorem{definition}{Definition}
\fi
\ifdefined\proposition
\else
  \newtheorem{proposition}{Proposition}
\fi
\makeatother
 \usepackage[textwidth=2in,colorinlistoftodos,prependcaption,textsize=footnotesize]{todonotes}
\expandafter\newif\csname ifGin@setpagesize\endcsname

\usepackage[normalem]{ulem}

 \usepackage{nicefrac}
\usepackage{natbib}
\let\cite\citep

\newcommand{\parencite}{\citep}
\newcommand{\textcite}{\citet}

\begin{document}
\maketitle
\begin{abstract}
  Neural replicator dynamics (NeuRD) is an alternative to the foundational softmax policy gradient (SPG) algorithm motivated by online learning and evolutionary game theory.
  The NeuRD expected update is designed to be nearly identical to that of SPG, however, we show that the Monte Carlo updates differ in a substantial way: the importance correction accounting for a sampled action is nullified in the SPG update, but not in the NeuRD update.
  Naturally, this causes the NeuRD update to have higher variance than its SPG counterpart.
  Building on implicit exploration algorithms in the adversarial bandit setting, we introduce capped implicit exploration (CIX) estimates that allow us to construct NeuRD-CIX, which interpolates between this aspect of NeuRD and SPG.
  We show how CIX estimates can be used in a black-box reduction to construct bandit algorithms with regret bounds that hold with high probability and the benefits this entails for NeuRD-CIX in sequential decision-making settings.
  Our analysis reveals a bias--variance tradeoff between SPG and NeuRD, and shows how theory predicts that NeuRD-CIX will perform well more consistently than NeuRD while retaining NeuRD's advantages over SPG in non-stationary environments.
\end{abstract}

\acknowledgements{
  During this project, Dustin Morrill, Esra'a Saleh, and Michael Bowling were supported by the Alberta Machine Intelligence Institute (Amii), CIFAR, and NSERC.
  Amy Greenwald is supported in part by NSF Award CMMI-1761546.
}
\startmain

\section{Introduction}

By design, the difference between the expected update procedures of softmax policy gradient (SPG; \citealp{Williams1992reinforce,Sutton00PolicyGradient}) and neural replicator dynamics (NeuRD; \citealp{hennes2020neurd}) is very small; the SPG update is effectively a rescaling of the NeuRD update based on the SPG agent's current policy.
Their Monte Carlo updates however, have an important difference: the importance correction accounting for a sampled action is nullified in the SPG update but not in the NeuRD update, naturally leading the NeuRD update to have larger variance.
In this work, we introduce a continuous bias--variance tradeoff between SPG and NeuRD via the NeuRD with capped implicit exploration (CIX) algorithm.
Implicit exploration~\parencite{kocak2014ix} is a technique for carefully adding bias to importance corrected utility function estimates in the adversarial bandit setting and we modify this idea by introducing a cap on the amount of bias that can be added when the sampled action probability is already large.
We show that our capped implicit exploration (CIX) estimates facilitate the easy construction of bandit algorithms with regret bounds that hold with high probability via a black box reduction to full-monitoring regret minimization.
These high probability regret bounds in the bandit setting imply high probability optimality and regret bounds for NeuRD-CIX in sequential decision-making settings.

\section{Background}

\subsection{Partially Observable History Processes}

A \emph{partially observable history process} (\emph{POHP}; ~\citealp{pohp}) models the perspective of an agent in a domain much larger and more complex than themself.
The titular \emph{history} is a record of \emph{actions}.
The finite set of actions allowed after each history $h \in \Histories$ is determined by $\Actions(h)$.
The agent's limitations in controlling the environment through their actions is personified as a \emph{daimon}, who chooses an action after each of the agent's choices sampled from a history-dependent \emph{policy}, $\daimonStrat \in \DaimonStrategySet$, $\daimonStrat(h) \in \simplex\subex*{\Actions(h)}$,
where $\simplex\subex*{\Actions(h)}$ is the probability simplex over the set of actions immediately following \emph{passive history} $h \in \Histories_{\ObservationSet}$ where the agent has acted last.
An \emph{observation function},
$\ObservationFn : \Histories_{\Actions} \to \ObservationSet$
models partial observability of the daimon's actions, where $\Histories_{\Actions}$ represents the set of \emph{active histories} where the agent is next to act and $\ObservationSet$ is the set of possible observations.\footnote{The union of the passive and active histories forms the complete set of histories, \ie/, $\Histories = \Histories_{\ObservationSet} \cup \Histories_{\Actions}$.}
The process terminates at history $h \in \Histories_{\Actions}$ with probability $1 - \gamma(h)$ where $\gamma : \Histories_{\Actions} \to [0, 1]$ is a \emph{continuation probability function}.

The agent is responsible for updating their own internal \emph{agent state}, $\agentState \in \AgentStateSet$, based on their actions and observations.
We summarize the agent's internal update mechanisms as an \emph{update function}
$\updateFn : \Histories \to \AgentStateSet$.
The agent state is conditioning information for the agent's \emph{policy} $\policy \in \PolicySet$ where
$\policy(\agentState) \in \simplex(\Actions(\agentState))$,
$\abs{\Actions(\agentState)} \le \infty$,
and
$\Actions(\agentState) = \Actions(h)$
for any active history $h$ in agent state $\agentState$'s \emph{information set},
$\infoSet(\agentState) = \set{ h' \where \updateFn(h') = \agentState }$.
The agent accumulates \emph{rewards} from a bounded \emph{reward function},
$\RewardFn : \ObservationSet \to [-\maxReward, 0]$.
We assume rewards are non-positive without loss of generality.
The \emph{return} (cumulative reward) that the agent acquires from active history $h \in \Histories_{\Actions}$ is
$\RandomReturn_h(\strat; \daimonStrat)
  = \sum_{i = 1}^{\infty}
    Y_i \RewardFn\big( \ObservationFn(H_i) \big)$,
where the initial history in the trajectory is $H_1 = h$,
the agent's action on each step is $A_i \sim \strat\big( \updateFn(H_i) \big)$,
the daimon's action on each step is $B_i \sim \daimonStrat(H_i A_i)$,
the history is updated as the concatenation $H_{i + 1} = H_i A_i B_i$, and
the continuation indicator is the product $Y_{i + 1} = Y_i \Gamma_i \in \set{0, 1}$ with $Y_1 = 1$ and $\Gamma_i \sim \gamma(H_i)$.

To make a POHP easier to work with, one of two different assumptions are often made.

A POHP could have a \emph{finite horizon}, where each history is a prefix of a \emph{terminal history}, $z \in \TerminalHistories \subseteq \Histories_{\Actions}$, where $\gamma(z) = 0$, which aligns the POHP with a given player's view of an extensive-form game (see, \eg/, \textcite{osborne94}).
This is also often combined with the assumption that the agent is not \emph{absent-minded}, that is, the agent never encounters the same agent state more than once during their lifetime.\footnote{The agent can ensure this by keeping track of how many actions they have taken. Such updates are said to be \emph{timed}.}
In this case, the \emph{realization probability} of an agent state can be defined as the sum of \emph{reach probabilities} across histories in that state's information set.
Formally,
$\Prob_{\strat, \daimonStrat}[\agentState]
  = \sum_{h \in \infoSet(\agentState)}
    \Prob_{\strat, \daimonStrat}[h]$
where
$\Prob_{\strat, \daimonStrat}[h] =
  \prod_{i = 1}^{\abs{h}} \Prob_{\strat, \daimonStrat}[h_i \given h_{< i}]$,
$h_i$ is the action at the $i^{\text{th}}$ position in the history,
$h_{< i}$ is the $i$-length prefix of $h$ with $h_{< 1} = \emptyHistory$ (the empty history), and
$
  \Prob_{\strat, \daimonStrat}[h_i \given h_{< i}]
    =
      \gamma(h_{< i}) \strat\subex*{h_i \given \updateFn(h_{< i})}
$ if $h_{< i} \in \Histories_{\Actions}$
and $\daimonStrat(h_i \given h_{< i})$ otherwise.

Alternatively, a POHP can represent a \emph{discounted Markov decision process} (\emph{MDP}) by assuming that (i) the daimon chooses actions according to its own state, (ii) the observation function reveals the daimon's state, and (iii) the continuation probability is not always one.
In this case, the daimon's state reproduces the MDP state, the daimon's policy $\daimonStrat$ is the MDP state \emph{transition distribution}, and continuation probabilities represent \emph{discount factors}.
In this case, the probability of realizing a given agent state $\agentState$ can be defined by marginalizing the probability of transitioning to $\agentState$ in $k$ actions, \ie/,\footnote{The $k$-step transition distribution is derived from the agent's policy and the transition distribution.}
$
  \Prob_{\strat, \daimonStrat}[\agentState]
    = \sum_{\bar{\agentState} \in \AgentStateSet}
      \sum_{k = 0}^{\infty}
        \Prob_{\strat, \daimonStrat}[\agentState, k \given \bar{\agentState}]
          \Prob_{\strat, \daimonStrat}[\bar{\agentState} \given \updateFn(\emptyHistory)].
$
For POHPs that are repeatedly evaluated, we assume that the continuation probabilities are set to ensure the POHP terminates almost surely, \eg/, by setting $\gamma(\cdot) < 1$.

Agent state realization probabilities give rise to a \emph{belief},
$\belief_{\agentState}^{\strat, \daimonStrat}:
  h \mapsto
    \Prob_{\strat, \daimonStrat}[h] / \Prob_{\strat, \daimonStrat}[\agentState]$,
about the relative likelihood of each history in the information set of each agent state $\agentState$,
and the expected return from $\agentState$, denoted
$\stateValueFn_{\agentState}(\strat; \daimonStrat)
  = \E_{H \sim \belief_{\agentState}^{\strat, \daimonStrat}}\subblock*{
    \RandomReturn_H(\strat; \daimonStrat)
  }$.
The \emph{action value} function is the expected return from $\agentState$ given an action, \ie/,
$\actionValueFn_{\agentState}(a, \strat; \daimonStrat)$,
and the \emph{advantage} of action $a$ is the difference
$
  \regret_{\agentState}(a, \strat; \daimonStrat)
    =
      \actionValueFn_{\agentState}\subex*{a, \strat; \daimonStrat}
      - \stateValueFn_{\agentState}\subex*{\strat; \daimonStrat}
$.

\subsection{The Bandit Setting and Implicit Exploration}

Consider a repeated POHP where, on each round $t$, the agent chooses a \emph{pure (deterministic) policy}, $\PureStrat^t \in \PureStratSet$ and then receives the expected return for that strategy,
$\utility(\PureStrat^t; \daimonStrat^t) = \E\subblock*{\RandomReturn_{\emptyHistory}(\PureStrat^t; \daimonStrat^t)}$,
by playing $\PureStrat^t$ out in the POHP with the daimon's policy on round $t$, $\daimonStrat^t$.
This describes an (adversarial) \emph{bandit} problem, where the set of pure policies is a finite set of ``bandit arms'' and the daimon determines the agent's payoff function on each round through their policy.
The \emph{full monitoring} version of the bandit problem is a simplification where the agent instead receives a return for each pure policy on each round.
Often, bandit algorithms can be constructed via a black-box reduction to full monitoring.

Typically, the goal is for the agent to perform nearly as well as the best pure policy in hindsight, that is, to ensure sublinear \emph{cumulative regret}:
$\regret^{1:T}(\pureStrat)
  = \sum_{t = 1}^T \regret(\pureStrat, \PureStrat^t; \daimonStrat^t) \in \smallo{T}$
for all $\pureStrat \in \PureStrategySet$, where the \emph{instantaneous regret} is the difference
$\regret(\pureStrat, \PureStrat^t; \daimonStrat^t)
  = \utility(\pureStrat; \daimonStrat^t) - \utility(\PureStrat^t; \daimonStrat^t)$.
We call agents that achieve this condition \emph{hindsight rational} since they achieve approximate rationality (optimality) in hindsight over their lifetime~\parencite{hsr2020}.
In general, agents must randomize to be hindsight rational, so we assume that $\PureStrat^t$ is sampled from a \emph{mixed policy}
$\strat^t \in \simplex(\PureStrategySet)$
on each round $t$, and we focus on achieving hindsight rationality with high probability under this sampling procedure.

Implicit exploration~\parencite{kocak2014ix} is a method for estimating unobserved payoffs in the bandit setting that controls estimation variance by introducing some bias.
Given pure policy $\PureStrat \sim \strat$, the implicit exploration estimated utility function is
$\est{\utility}: \pureStrat; \PureStrat, \ixParam, \strat, \daimonStrat \mapsto \frac{
  \ind{\PureStrat = \pureStrat}\utility(\pureStrat; \daimonStrat)
}{
  \strat(\pureStrat) + \ixParam
}$
for an implicit exploration parameter $\ixParam \in [0, 1]$.
Thanks to the fact that payoffs are non-positive so that large payoffs are closer to zero, $\est{\utility}$ always overestimates payoffs, which encourages greater exploration.
Implicit exploration was originally used to construct the Exp3-IX algorithm~\parencite{kocak2014ix,neu2015moreIx}.

\section{SPG and NeuRD Monte-Carlo Updates}

SPG updates a parameterized policy over time.
After $t$ updates, SPG deploys policy
$\policy^{t + 1}:
  \agentState \mapsto
    \exp\subex*{f(\agentState, \cdot; \odpDecision^{t + 1})} / \ip{\ones}{\exp\subex*{f(\agentState, \cdot; \odpDecision^{t + 1})}}$
with $d$ parameters
$\odpDecision^{t + 1} \in \reals^d$
and function approximator
$f: \AgentStateSet \times \Actions(\cdot) \times \reals^d \to \reals$,
where $\Actions(\cdot)$ represents the set of allowed actions given the first argument of $f$
and $\ip{\cdot}{\cdot}$ is the vector inner-product operator.
The Monte-Carlo or REINFORCE~\parencite{Williams1992reinforce} SPG update generates $\odpDecision^{t + 1}$ by changing $\odpDecision^t$ according to stochastic gradient descent on a Monte Carlo estimate of the expected return, with $\odpDecision^1$ arbitrarily.
Let $H$ be a history in a random trajectory sampled by playing out $\strat^t$ and $\daimonStrat^t$ and let $A \sim \strat^t(\updateFn(H))$.
The Monte-Carlo SPG update is then
\begin{align}
  \odpDecision^{t + 1}
    &= \odpDecision^t
      + \stepSize^t \overbrace{
        \dfrac{
          1
        }{
          \strat^t\subex*{A \given \updateFn(H)}
        }
      }^{\text{Importance correction.}}
        \RandomReturn_{HA}(\strat^t; \daimonStrat^t)
        \grad_{\odpDecision^t} \strat^t\subex*{A \given \updateFn(H)}\\
    &= \odpDecision^t
      + \stepSize^t
        \dfrac{
          \strat^t\subex*{A \given \updateFn(H)}
        }{
          \strat^t\subex*{A \given \updateFn(H)}
        }
        \RandomReturn_{HA}(\strat^t; \daimonStrat^t)
          \sum_{a \in \Actions(H)}
            \subex*{
              \ind{a = A}
              - \strat^t\subex*{a \given \updateFn(H)}
            }
              \grad_{\odpDecision^t} f\subex*{\updateFn(H), a; \odpDecision^t}\\
    &= \odpDecision^t
      + \stepSize^t
      \RandomReturn_{HA}(\strat^t; \daimonStrat^t)
        \sum_{a \in \Actions(H)}
          \subex*{
            \ind{a = A}
            - \strat^t\subex*{a \given \updateFn(H)}
          }
            \grad_{\odpDecision^t} f\subex*{\updateFn(H), a; \odpDecision^t}.
    \label{eq:mcSpgUpdate}
\end{align}

NeuRD uses the same policy parameterization, but its update is based on the gradient of a distance constructed from a return rather than the return alone.
Let
$y^t_{\updateFn(H), a} = \stopGrad{f\subex*{\updateFn(H), a; \odpDecision^t}}$
be the NeuRD preference value for action $a$ on sampled history $H$ (conditioned only on agent state $\updateFn(H)$ and not the underlying history), where $y^t_{\updateFn(H), a}$'s dependence on $\odpDecision^t$ is blocked by the ``stop gradient'' operator $\stopGrad{\cdot}$.
Let
$\regret_{HA}(a, \strat^t; \daimonStrat^t)
  = \frac{\RandomReturn_{HA}(\strat^t; \daimonStrat^t)}{\strat^t\subex*{A \given \updateFn(H)}} \subex*{ \ind{a = A} - \strat^t\subex*{A \given \updateFn(H)} }$
be an importance corrected Monte-Carlo estimate of the advantage from history $HA$ following sampled action $A \sim \strat^t\subex*{\updateFn(H)}$.
The Monte-Carlo NeuRD update objective is the sum of squared differences between
$y^t_{\updateFn(H), a} + \stepSize^t \regret_{HA}(a, \strat^t; \daimonStrat^t)$
and
$f\subex*{\updateFn(H), a; \odpDecision^t}$
across each action $a$.\footnote{$y^t_{\updateFn(H), a}$ bootstraps ``into the past'', estimating the cumulative advantage for action $a$ across $t - 1$ past updates at agent state $\updateFn(H)$.}
The update itself is then
\begin{align}
  \odpDecision^{t + 1}
    &= \odpDecision^t
      - \sum_{a \in \Actions(H)}
        \grad_{\odpDecision^t} \dfrac{1}{2} \subex*{
          y^t_{\updateFn(H), a} + \stepSize^t \regret_{HA}(a, \strat^t; \daimonStrat^t)
          - f\subex*{\updateFn(H), a; \odpDecision^t}
        }^2\\
    &= \odpDecision^t
      + \stepSize^t
        \sum_{a \in \Actions(H)}
          \regret_{HA}(a, \strat^t; \daimonStrat^t)
          \grad_{\odpDecision^t} f\subex*{\updateFn(H), a; \odpDecision^t}\\
    &= \odpDecision^t
      + \stepSize^t
        \dfrac{\RandomReturn_{HA}(\strat^t; \daimonStrat^t)}{\strat^t\subex*{A \given \updateFn(H)}}
          \sum_{a \in \Actions(H)}
            \subex*{ \ind{a = A} - \strat^t\subex*{A \given \updateFn(H)} }
              \grad_{\odpDecision^t} f\subex*{\updateFn(H), a; \odpDecision^t}.
    \label{eq:mcNeurdUpdate}
\end{align}

There are two differences between \cref{eq:mcSpgUpdate} and \cref{eq:mcNeurdUpdate}: the sampled action's importance correction is nullified in the SPG update but not the NeuRD update, and each unsampled action $a \ne A$ contributes
$-\strat^t\subex*{a \given \updateFn(H)} \grad_{\odpDecision^t} f\subex*{\updateFn(H), a; \odpDecision^t}$
in the SPG update instead of
$-\strat^t\subex*{A \given \updateFn(H)} \grad_{\odpDecision^t} f\subex*{\updateFn(H), A; \odpDecision^t}$
in the NeuRD update.
Since actions with larger probabilities are sampled more often, the latter contributions will tend to have a larger magnitude in the NeuRD update, but the practical significance of this difference is unclear.
However, the difference in the influence of the importance correction will clearly cause the NeuRD update to have higher variance than the SPG update.
We now describe the idea of capped implicit exploration and its application to NeuRD where the influence of the importance correction can be tuned between these two extremes.
We will also present theory to motivate the use of an intermediate amount of importance correction influence in practice.

\section{Capped Implicit Exploration}

Considering a repeated POHP bandit setting again, given pure policy $\PureStrat \sim \strat$ sampled from mixed policy $\strat$, the \emph{capped implicit exploration} (\emph{CIX}) estimate of the utility function is
$\est{\utility}: \pureStrat; \strat, \daimonStrat, \PureStrat, \ixParam \mapsto
  \ind{\PureStrat = \pureStrat}\utility(\pureStrat; \daimonStrat) / \CixCorrection(\PureStrat; \strat, \ixParam)$
where
$\CixCorrection(\PureStrat; \strat, \ixParam) = \min\set{1, \strat(\PureStrat) + \ixParam}$
for an implicit exploration parameter $\ixParam \in [0, 1]$.
The cap on the denominator potentially removes bias when $\strat(\PureStrat)$ is already large.
The same regret bound can be proven for CIX as the original implicit exploration using largely the same proof as that presented in Chapter 12 of \textcite{lattimore2020bandit}, though here we present the result as a black-box reduction to full monitoring.
\begin{theorem}
  \label{thm:cix}
  Assume that expected return is bounded in $[-\maxReturn, 0]$.
If, on each round $t \ge 1$, the CIX utility function with parameter $\ixParam^t \in (0, 1]$,
$\est{\utility}^t: \pureStrat \mapsto \est{\utility}(\pureStrat; \strat^t, \daimonStrat^t, \PureStrat^t, \ixParam^t)$,
is given to a full monitoring algorithm that chooses mixed policy $\strat^t$, then the bandit algorithm that chooses pure policy $\PureStrat^t \sim \strat^t$ upper bounds the cumulative regret for not choosing pure policy $\pureStrat$ on all $T$ rounds, with probability $1 - \delta$, as
$\regret^{1:T}(\pureStrat)
    \le
      g(T, \abs{\PureStrategySet}, \maxReturn, \delta, \tuple{\ixParam^t}_{t = 1}^T)
      + h(T, \abs{\PureStrategySet}, \maxReturn, \delta, \tuple{\ixParam^t}_{t = 1}^T)$,
where $g(T, \abs{\PureStrategySet}, \maxReturn, \delta, \tuple{\ixParam^t}_{t = 1}^T)$ is the algorithm's expected cumulative regret under the CIX utility functions and
\begin{align*}
  h(T, \abs{\PureStrategySet}, \maxReturn, \delta, \tuple{\ixParam^t}_{t = 1}^T)
    = \maxReturn \abs{\PureStrategySet} \sum_{t = 1}^T \ixParam^t
      + \dfrac{\maxReturn}{2\ixParam^{\textsc{min}}} \log\subex*{\dfrac{\abs{\PureStrategySet} + 1}{\delta}}
      + \dfrac{\maxReturn}{2} \log\subex*{\dfrac{\abs{\PureStrategySet} + 1}{\delta}}
\end{align*}
is slack to account for the regret that could be forced according to the statistical properties of CIX estimates, given $\ixParam^{\textsc{min}} = \min\set{\ixParam^t}_{t = 1}^T$.
If
$\ixParam^t
  = \xi
    \sqrt{
      \frac{1}{\abs{\PureStrategySet} t}
    }$
for $\xi > 0$, then with probability $1 - \delta$,
\begin{align*}
  h(T, \abs{\PureStrategySet}, \maxReturn, \delta, \tuple{\ixParam^t}_{t = 1}^T)
    &=
      \subex*{
        2 \xi
        + \dfrac{1}{2\xi}
          \log\subex*{\dfrac{\abs{\PureStrategySet} + 1}{\delta}}
      } \maxReturn \sqrt{\abs{\PureStrategySet} T}
    + \dfrac{\maxReturn}{2} \log\subex*{\dfrac{\abs{\PureStrategySet} + 1}{\delta}} \in \bigO{\sqrt{T}}.
\end{align*}
If
$g(T, \abs{\PureStrategySet}, \maxReturn, \delta, \tuple{\ixParam^t}_{t = 1}^T) \in \smallo{T}$, then $\regret^{1:T}(\pureStrat) \in \smallo{T}$ with probability $1 - \delta$. \end{theorem}
\textbf{Proof deferred to \cref{appendix}.}
The scaling factor $\xi$ could be chosen by optimizing the bound and a confidence-dependent bound could be achieved
by following Chapter 12 of \textcite{lattimore2020bandit}.

\emph{NeuRD-CIX} applies this idea to the action decisions at each agent state.
\begin{definition}
  Define
  $\est{\actionValueFn}_h: a, \strat; \daimonStrat, A, \ixParam \mapsto
    \ind{A = a}\RandomReturn_{hA}\subex*{\strat; \daimonStrat}/\CixCorrection_h(A; \strat, \ixParam)$
  as the CIX estimate of the action value function at active history $h$ given a random return from random action $A \sim \strat\subex*{\updateFn(h)}$,
  where
  $\CixCorrection_h(A; \strat, \ixParam) = \min\set*{1, \strat\subex*{a \given \updateFn(h)} + \ixParam}$.
  The CIX estimate of the advantage function is then
  \begin{align*}
    \est{\regret}_{hA}: a, \strat; \daimonStrat, \ixParam \mapsto
      \dfrac{
        \RandomReturn_{hA}\subex*{\strat; \daimonStrat}
      }{
        \CixCorrection_h(A; \strat, \ixParam)
      }
      \subex*{
        \ind{a = A} - \strat\subex*{A \given \updateFn(h)}
      }.
  \end{align*}
  Finally, after $t$ updates, the next NeuRD-CIX update on sampled history $H$ is
  \begin{align}
    \odpDecision^{t + 1}
      &= \odpDecision^t
        - \sum_{a \in \Actions(H)}
            \grad_{\odpDecision^t} \dfrac{1}{2} \subex*{
              y^t_{\updateFn(H), a} + \stepSize^t \est{\regret}_{HA}(a, \strat^t; \daimonStrat^t, \ixParam^t)
              - f\subex*{\updateFn(H), a; \odpDecision^t}
            }^2\\
      &= \odpDecision^t
        + \stepSize^t
          \sum_{a \in \Actions(H)}
            \est{\regret}_{HA}(a, \strat^t; \daimonStrat^t, \ixParam^t)
              \grad_{\odpDecision^t} f\subex*{\updateFn(H), a; \odpDecision^t}\\
      &= \odpDecision^t
        + \stepSize^t
          \dfrac{
            \RandomReturn_{HA}\subex*{\strat^t; \daimonStrat^t}
          }{
            \CixCorrection_H(A; \strat^t, \ixParam^t)
          }
            \sum_{a \in \Actions(H)}
              \subex*{
                \ind{a = A} - \strat^t\subex*{A \given \updateFn(h)}
              }
                \grad_{\odpDecision^t} f\subex*{\updateFn(H), a; \odpDecision^t}.
      \label{eq:mcNeurdCixUpdate}
  \end{align}
\end{definition}
If $\ixParam^t = 0$ for all updates $t$, NeuRD-CIX reproduces NeuRD (\cref{eq:mcNeurdCixUpdate} becomes \cref{eq:mcNeurdUpdate}).
If $\ixParam^t \ge 1$ instead, \cref{eq:mcNeurdCixUpdate} becomes
\begin{align}
  \odpDecision^{t + 1}
    &= \odpDecision^t
      + \stepSize^t
        \RandomReturn_{HA}\subex*{\strat^t; \daimonStrat^t}
          \sum_{a \in \Actions(H)}
            \subex*{
              \ind{a = A} - \strat^t\subex*{A \given \updateFn(h)}
            }
              \grad_{\odpDecision^t} f\subex*{\updateFn(H), a; \odpDecision^t},
  \label{eq:mcNeurdCixOneUpdate}
\end{align}
where the only difference between \cref{eq:mcNeurdCixOneUpdate} and the SPG update (\cref{eq:mcSpgUpdate}) is that each unsampled action $a \ne A$ contributes
$-\strat^t\subex*{A \given \updateFn(h)}
  \grad_{\odpDecision^t} f\subex*{\updateFn(H), a; \odpDecision^t}$
rather than
$-\strat^t\subex*{a \given \updateFn(h)}
  \grad_{\odpDecision^t} f\subex*{\updateFn(H), a; \odpDecision^t}$.

\section{Experiments}

\begin{figure}[t]
  \centering
  \begin{minipage}[t]{0.48\textwidth}
    \includegraphics[width=\textwidth]{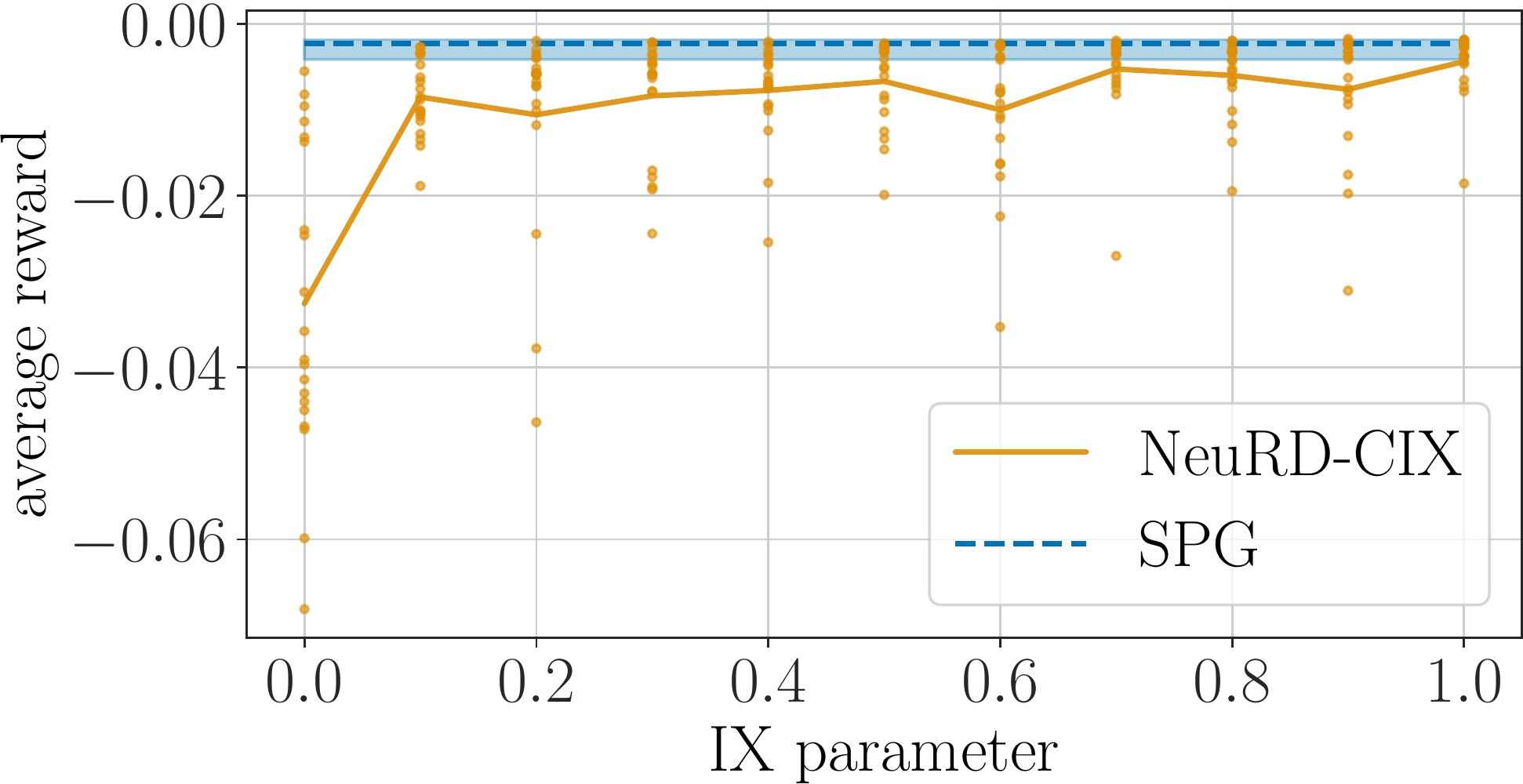}
  \end{minipage}\hfill \begin{minipage}[t]{0.48\textwidth}
    \includegraphics[width=\textwidth]{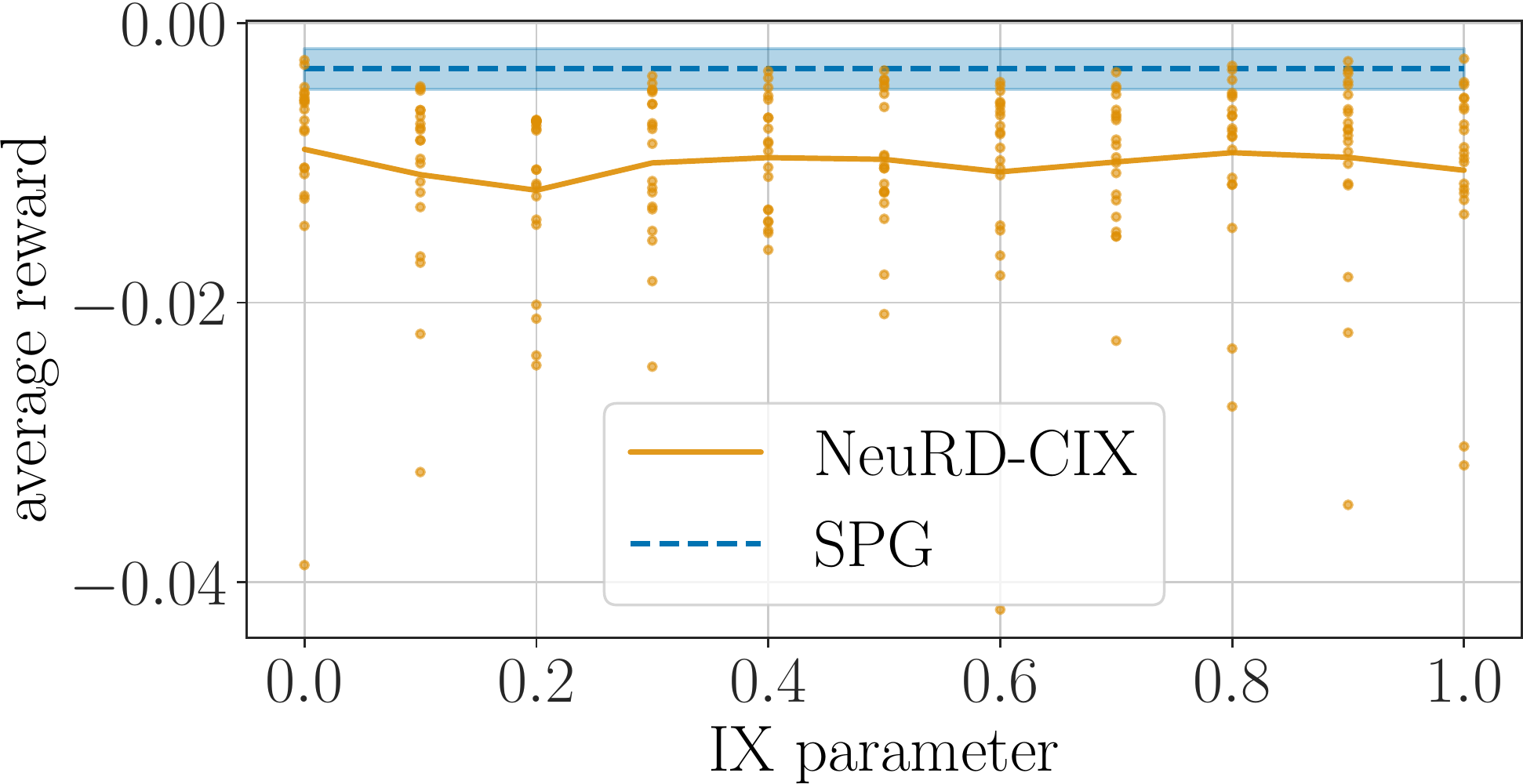}
  \end{minipage}
  \caption{Performance over $300,000$ steps in catch (left) and cart pole (right).
    Each dot is a NeuRD-CIX run and the solid line is the average across all $20$.
    The dashed line is SPG's average performance across $20$ runs and shading shows its range from best to worst.}
  \label{fig:catchAndCartPoleEnvResults}
\end{figure}

Can NeuRD-CIX be a drop-in replacement for SPG using an \emph{actor-critic}~\parencite{barto1983actorCritic} architecture?

We compare accumulated reward in \emph{catch} (from \texttt{bsuite}; \citealp{osband2020bsuite}) and \emph{cart pole} (from OpenAI Gym; \citealp{openAiGym}).
In catch, a ball is dropped from a high place and the agent must move a basket to catch the ball.
The agent receives a reward of $0$ unless they fail to catch the falling ball, then they receive $-1$.
As soon as the ball lands, another ball is dropped.
In cart pole, the agent must balance a pole by moving a cart along a single axis.
The agent receives a reward of $0$ unless the pole falls, then they receive $-1$ and the pole is reset to a near upright position.

SPG and NeuRD-CIX share the same neural actor-critic architecture.
Both actor and critic networks are linear functions over the same $2$-layer neural network embedding.
The embedding uses $256$ rectified linear units (ReLUs) plus skip connections, followed by $256$ ReLUs.
The critic predicts the $32$-step TD($\lambda= 0.9$) state-value returns.
There is no discounting for the first $32$ steps and a discount of $0.9$ on values for bootstrapping.
All neural network parameters are updated after every step with Adam~\parencite{kingma2014adam} and backgropagation, without gradient clipping or logit thresholding~\parencite{hennes2020neurd}.\footnote{The non-learning rate Adam parameters were chosen as $\beta_1 = 0$ since updates are online and $\beta_2 = 0.999$ to follow convention.}
We tried a range of Adam learning rates for SPG in catch and cart pole, and identified the best learning rate according to the average reward accumulated after $300,000$ steps and $20$ runs ($0.001797$ in catch and $0.00005$ in cart pole).\footnote{The learning rates tried in catch and cart pole were
\[\set*{
  \begin{aligned}
    &0.00001, 0.00002, 0.00004, 0.00005, 0.00008, 0.00008, 0.00010, 0.00014, 0.00016, 0.00017,\\
    &0.00023, 0.00028, 0.00032, 0.00039, 0.00046, 0.00063, 0.00065, 0.00077, 0.00108, 0.00126,\\
    &0.00129, 0.00180, 0.00215, 0.00251, 0.00300, 0.00359, 0.00500, 0.00599, 0.01000
  \end{aligned}
}\].
}
For all NeuRD-CIX runs, the Adam learning rate is set to the best SPG learning rate for the given environment.
This experimental setup ensures a favorable comparison for the SPG baseline.

The results are presented in \cref{fig:catchAndCartPoleEnvResults} as sensitivity plots of average reward achieved by NeuRD-CIX across different fixed implicit exploration (IX) parameter values ($\ixParam^t = \ixParam$).

In these single-agent environments, roughly the best performance is achieved with full implicit exploration ($\ixParam = 1$).
In catch, NeuRD-CIX is clearly less reliable without any implicit exploration ($\ixParam = 0$) and achieves nearly the same performance as SPG with full implicit exploration.
In cart pole, the implicit exploration parameter does not appear to have much of an impact, NeuRD-CIX is substantially worse than SPG and NeuRD-CIX's performance is highly variable across parameter settings.

\section{Conclusion}

We developed capped implicit exploration (CIX) estimates that bias sampled utility values so as to induce exploration and reduce estimation variance.
We showed how CIX estimates can be used in a black-box reduction to full-monitoring regret minimization and how they can be used in sequential decision making within the NeuRD-CIX algorithm.
In addition, we showed that the NeuRD-CIX algorithm interpolates along a bias--variance tradeoff between Monte-Carlo softmax policy gradient (maximal bias, lower variance) and a Monte-Carlo version of the original NeuRD algorithm (zero bias, larger variance).
\cref{thm:cix} suggests that reliable performance can be achieved with a gradually decreasing amount of implicit exploration and bias.

NeuRD-CIX has theoretical benefits over SPG, but the practical comparison is less clear.
In an empirical comparison favorable to SPG, NeuRD-CIX is competitive with but worse than SPG.

Some questions are left open by this work.
Why does NeuRD-CIX($1$) perform worse than SPG?
What architecture changes can be made to improve NeuRD-CIX's performance?
How does the IX parameter impact performance in non-stationary environments?

{
  \bibliography{references}
}

\appendix
\section{Proof of \cref{thm:cix}}
\label{appendix}

The proof of \cref{thm:cix} requires Lemma 12.2 from \textcite{lattimore2020bandit}, a simplified version of which is restated here without proof for reference.
\begin{lemma}
  \label{lem:cramerChernoff}
  Let
  $\tuple{\alpha^t_{\pureStrat}}_{t \in [1, T], \pureStrat \in \PureStrategySet}$
  and
  $\tuple{\lambda^t_{\pureStrat}}_{t \in [1, T], \pureStrat \in \PureStrategySet}$
  be real-valued random variables associated with each round $t$ of an ODP.
  Ensure that each $\alpha^t_{\pureStrat}$ and $\lambda^t_{\pureStrat}$ depend only on information revealed rounds before $t$, \ie/, they must be predictable given rounds up to and including round $t - 1$.
  In addition, ensure that
  $0
    \le
      -\frac{\alpha^t_{\pureStrat}}{\maxReturn}
        \est{\utility}(\pureStrat; \strat^t, \daimonStrat^t, \PureStrat^t, 0)
    \le
      2 \lambda^t_{\pureStrat}$
  for each $t$ and $\pureStrat$.
  Then,
  \[
    \sum_{t = 1}^T \sum_{\pureStrat \in \PureStrategySet}
      \frac{\alpha^t_{\pureStrat}}{\maxReturn} \subex*{
        \utility(\pureStrat; \daimonStrat^t)
        - \dfrac{
          \est{\utility}(\pureStrat; \strat^t, \daimonStrat^t, \PureStrat^t, 0)
        }{
          1 + \lambda^t_{\pureStrat}
        }
      }
    \le
      \log\subex*{\dfrac{1}{\delta}}
  \]
  holds with probability $1 - \delta$ for $\delta \in (0, 1)$.
\end{lemma}

With \cref{lem:cramerChernoff}, we can prove \cref{thm:cix}.
\begin{theorem}
  Assume that expected return is bounded in $[-\maxReturn, 0]$.
If, on each round $t \ge 1$, the CIX utility function with parameter $\ixParam^t \in (0, 1]$,
$\est{\utility}^t: \pureStrat \mapsto \est{\utility}(\pureStrat; \strat^t, \daimonStrat^t, \PureStrat^t, \ixParam^t)$,
is given to a full monitoring algorithm that chooses mixed policy $\strat^t$, then the bandit algorithm that chooses pure policy $\PureStrat^t \sim \strat^t$ upper bounds the cumulative regret for not choosing pure policy $\pureStrat$ on all $T$ rounds, with probability $1 - \delta$, as
$\regret^{1:T}(\pureStrat)
    \le
      g(T, \abs{\PureStrategySet}, \maxReturn, \delta, \tuple{\ixParam^t}_{t = 1}^T)
      + h(T, \abs{\PureStrategySet}, \maxReturn, \delta, \tuple{\ixParam^t}_{t = 1}^T)$,
where $g(T, \abs{\PureStrategySet}, \maxReturn, \delta, \tuple{\ixParam^t}_{t = 1}^T)$ is the algorithm's expected cumulative regret under the CIX utility functions and
\begin{align*}
  h(T, \abs{\PureStrategySet}, \maxReturn, \delta, \tuple{\ixParam^t}_{t = 1}^T)
    = \maxReturn \abs{\PureStrategySet} \sum_{t = 1}^T \ixParam^t
      + \dfrac{\maxReturn}{2\ixParam^{\textsc{min}}} \log\subex*{\dfrac{\abs{\PureStrategySet} + 1}{\delta}}
      + \dfrac{\maxReturn}{2} \log\subex*{\dfrac{\abs{\PureStrategySet} + 1}{\delta}}
\end{align*}
is slack to account for the regret that could be forced according to the statistical properties of CIX estimates, given $\ixParam^{\textsc{min}} = \min\set{\ixParam^t}_{t = 1}^T$.
If
$\ixParam^t
  = \xi
    \sqrt{
      \frac{1}{\abs{\PureStrategySet} t}
    }$
for $\xi > 0$, then with probability $1 - \delta$,
\begin{align*}
  h(T, \abs{\PureStrategySet}, \maxReturn, \delta, \tuple{\ixParam^t}_{t = 1}^T)
    &=
      \subex*{
        2 \xi
        + \dfrac{1}{2\xi}
          \log\subex*{\dfrac{\abs{\PureStrategySet} + 1}{\delta}}
      } \maxReturn \sqrt{\abs{\PureStrategySet} T}
    + \dfrac{\maxReturn}{2} \log\subex*{\dfrac{\abs{\PureStrategySet} + 1}{\delta}} \in \bigO{\sqrt{T}}.
\end{align*}
If
$g(T, \abs{\PureStrategySet}, \maxReturn, \delta, \tuple{\ixParam^t}_{t = 1}^T) \in \smallo{T}$, then $\regret^{1:T}(\pureStrat) \in \smallo{T}$ with probability $1 - \delta$. \end{theorem}
\begin{proof}
  Let $\bs{\strat} = \tuple{ \strat^t }_{t = 1}^T$ and $\PureStrat = \tuple{ \PureStrat^t }_{t = 1}^T$ be the sequence of full monitoring and bandit policies generated by the algorithm across $T$ rounds.
  Let the expected cumulative payoff of $\bs{\strat}$ under the true utility function $\utility$ be
  $\utility^{1:T}(\bs{\strat}) = \sum_{t = 1}^T \utility\subex*{\strat^t; \daimonStrat^t}$
  and that under the CIX utility functions be
  $\est{\utility}^{1:T}(\bs{\strat}) = \sum_{t = 1}^T \est{\utility}^t\subex*{\strat^t}$.
  Since pure policies are pont-mass mixed policies,
  $\utility^{1:T}(\bs{\PureStrat}) = \sum_{t = 1}^T \utility\subex*{\PureStrat^t; \daimonStrat^t}$
  and
  $\est{\utility}^{1:T}(\bs{\PureStrat}) = \sum_{t = 1}^T \est{\utility}^t\subex*{\PureStrat^t}$.
  Overload this notation for a single pure policy as
  $\utility^{1:T}(\pureStrat) = \sum_{t = 1}^T \utility\subex*{\pureStrat; \daimonStrat^t}$.

  The proof works by separately bounding differences between expected and estimated values.
  Starting from the regret that the agent actually suffers, we bring together the necessary terms by adding and subtracting the same values as
  \begin{align}
    \regret^{1:T}(\pureStrat)
      &=
        \utility^{1:T}(\pureStrat) - \utility^{1:T}(\bs{\PureStrat})
        + \overbrace{
          \est{\utility}^{1:T}(\pureStrat) - \est{\utility}^{1:T}(\pureStrat)
        }^0
        + \overbrace{
          \est{\utility}^{1:T}(\bs{\strat}) - \est{\utility}^{1:T}(\bs{\strat})
        }^0\\
      &=
        \underbrace{
          \est{\utility}^{1:T}(\pureStrat) - \est{\utility}^{1:T}(\bs{\strat})
        }_{
          \text{Expected full monitoring regret under $\est{\utility}$.}
        }
        + \underbrace{
          \utility^{1:T}(\pureStrat) - \est{\utility}^{1:T}(\pureStrat)
        }_{
          \text{CIX estimation error for $\pureStrat$.}
        }
        + \underbrace{
          \utility^{1:T}(\bs{\PureStrat}) - \est{\utility}^{1:T}(\bs{\strat})
        }_{
          \text{CIX estimation error for $\bs{\PureStrat}$.}
        }\\
      &\le
        g(T, \abs{\PureStrategySet}, \delta, \tuple{\ixParam^t}_{t = 1}^T)
        + \utility^{1:T}(\pureStrat) - \est{\utility}^{1:T}(\pureStrat)
        + \sum_{t = 1}^T \utility(\PureStrat^t; \daimonStrat^t) - \est{\utility}^t(\strat^t).
  \end{align}

  We can simplify each
  $\utility(\PureStrat^t; \daimonStrat^t) - \est{\utility}^t(\strat^t)$
  using algebra and the fact that
  $\CixCorrection(\pureStrat'; \strat^t, \ixParam^t) - \strat^t(\pureStrat') \le \strat^t(\pureStrat') + \ixParam^t - \strat^t(\pureStrat') = \ixParam^t$
  for all pure policies
  $\pureStrat'$.
  \begin{align}
    \utility(\PureStrat^t; \daimonStrat^t) - \est{\utility}^t(\strat^t)
      &=
        \sum_{\pureStrat' \in \PureStrategySet}
          \ind{\PureStrat^t = \pureStrat'} \utility(\pureStrat'; \daimonStrat^t)
          - \strat^t(\pureStrat') \est{\utility}(\pureStrat'; \daimonStrat^t)\\
      &=
        \sum_{\pureStrat' \in \PureStrategySet}
          \subex*{
            1
            - \dfrac{\strat^t(\pureStrat')}{\CixCorrection(\pureStrat'; \strat^t, \ixParam^t)}
          }
          \ind{\PureStrat^t = \pureStrat'} \utility(\pureStrat'; \daimonStrat^t)\\
      &=
        \sum_{\pureStrat' \in \PureStrategySet}
          \dfrac{
            \CixCorrection(\pureStrat'; \strat^t, \ixParam^t)
            - \strat^t(\pureStrat')
          }{\CixCorrection(\pureStrat'; \strat^t, \ixParam^t)}
          \ind{\PureStrat^t = \pureStrat'} \utility(\pureStrat'; \daimonStrat^t)\\
      &\le
        \sum_{\pureStrat' \in \PureStrategySet}
          \dfrac{
            \ixParam^t
          }{\CixCorrection(\pureStrat'; \strat^t, \ixParam^t)}
          \ind{\PureStrat^t = \pureStrat'} \utility(\pureStrat'; \daimonStrat^t)\\
      &=
        \ixParam^t \sum_{\pureStrat' \in \PureStrategySet}
          \est{\utility}^t(\pureStrat').
  \end{align}
  $\est{\utility}^t(\pureStrat')$ also needs to be compared to true payoffs, so we add and subtract $\utility(\pureStrat'; \daimonStrat^t)$ so that
  \begin{align}
    \sum_{t = 1}^T \utility(\PureStrat^t; \daimonStrat^t) - \est{\utility}^t(\strat^t)
      &\le
        \sum_{t = 1}^T
          \ixParam^t
            \sum_{\pureStrat' \in \PureStrategySet}
              \est{\utility}^t(\pureStrat')
              + \overbrace{
                \utility(\pureStrat'; \daimonStrat^t)
                - \utility(\pureStrat'; \daimonStrat^t)
              }^0\\
      &=
        \sum_{t = 1}^T
          \ixParam^t
            \sum_{\pureStrat' \in \PureStrategySet}
              \subex*{-\utility(\pureStrat'; \daimonStrat^t)}
          + \ixParam^t
            \sum_{\pureStrat' \in \PureStrategySet}
              \utility(\pureStrat'; \daimonStrat^t) - \est{\utility}^t(\pureStrat')\\
      &\le
        \maxReturn \abs{\PureStrategySet} \sum_{t = 1}^T
          \ixParam^t
        + \sum_{t = 1}^T
          \ixParam^t
            \sum_{\pureStrat' \in \PureStrategySet}
              \utility(\pureStrat'; \daimonStrat^t) - \est{\utility}^t(\pureStrat'),
  \end{align}
  where the second inequality follows from the fact that $-\utility(\pureStrat'; \daimonStrat^t) \le \maxReturn$.

  To summarize, our bound is now
  \begin{align}
    \regret^{1:T}(\pureStrat)
      &\le
        g(T, \abs{\PureStrategySet}, \delta, \tuple{\ixParam^t}_{t = 1}^T)
        + \maxReturn \abs{\PureStrategySet} \sum_{t = 1}^T
          \ixParam^t
        + \utility^{1:T}(\pureStrat) - \est{\utility}^{1:T}(\pureStrat)
        + \sum_{t = 1}^T
          \ixParam^t
            \sum_{\pureStrat' \in \PureStrategySet}
              \utility(\pureStrat'; \daimonStrat^t) - \est{\utility}^t(\pureStrat').
  \end{align}

  Consider the difference
  \begin{align}
    \utility^{1:T}(\pureStrat) - \est{\utility}^{1:T}(\pureStrat)
      &= \sum_{t = 1}^T \utility(\pureStrat; \daimonStrat^t) - \est{\utility}^t(\pureStrat).
  \end{align}
  On each round $t$ where
  $\CixCorrection(\PureStrat^t; \strat^t, \ixParam^t) = 1$
  (those where $\strat^t(\PureStrat^t) + \ixParam^t \ge 1$),
  $\est{\utility}^t(\pureStrat)$ is equal to either $\utility^{1:T}(\pureStrat)$ or zero, depending on whether or not $\PureStrat^t = \pureStrat$.
  Since rewards are non-positive,
  $\utility(\pureStrat; \daimonStrat^t) - \est{\utility}^t(\pureStrat) \le 0$
  on each of these rounds.
  \def\uncappedRounds{\mathcal{T}}
  Let
  $\uncappedRounds = \set*{t \in \set{1, \ldots, T} \where \CixCorrection(\PureStrat^t; \strat^t, \ixParam^t) < 1}$
  be these rounds where the implicit exploration adjustment is not capped.
  Therefore,
  \begin{align}
    \utility^{1:T}(\pureStrat) - \est{\utility}^{1:T}(\pureStrat)
      &\le \sum_{t \in \uncappedRounds}
        \utility(\pureStrat; \daimonStrat^t) - \est{\utility}^t(\pureStrat)\\
      &\le \sum_{t \in \uncappedRounds}
        \utility(\pureStrat; \daimonStrat^t) - \est{\utility}^t(\pureStrat)
  \end{align}
  Now the steps in the proof of Lemma 12.3 of \textcite{lattimore2020bandit} can be repeated exactly to apply \cref{lem:cramerChernoff}, except accounting for a dynamic implicit exploration parameter.
  \begin{align}
    \utility^{1:T}(\pureStrat) - \est{\utility}^{1:T}(\pureStrat)
      &\le \sum_{t \in \uncappedRounds}
        \utility(\pureStrat; \daimonStrat^t) - \est{\utility}^t(\pureStrat)\\
      &\le \dfrac{1}{2 \ixParam^{\textsc{min}}}
        \sum_{t \in \uncappedRounds}
          2\ixParam^t
            \subex*{
              \utility(\pureStrat; \daimonStrat^t)
              - \dfrac{1}{
                1 + \dfrac{\ixParam^t}{\strat^t(\pureStrat)}
              } \dfrac{
                \ind{\PureStrat^t = \pureStrat} \utility(\pureStrat; \daimonStrat^t)
              }{
                \strat^t(\pureStrat)
              }
            }\\
      &= \dfrac{\maxReturn}{2 \ixParam^{\textsc{min}}}
        \sum_{t \in \uncappedRounds}
          \dfrac{2\ixParam^t}{\maxReturn}
            \subex*{
              \utility(\pureStrat; \daimonStrat^t)
              - \dfrac{1}{
                1 + \dfrac{\ixParam^t}{\strat^t(\pureStrat)}
              }
              \hat{\utility}^t(\pureStrat; \strat^t, \daimonStrat^t, \PureStrat^t, 0)
            }.
    \shortintertext{We can now apply \cref{lem:cramerChernoff} where
      $\lambda^t_{\pureStrat'} = \frac{\ixParam^t}{\strat^t(\pureStrat)}$
      and
      $\alpha^t_{\pureStrat'} = 2 \ind{\pureStrat' = \pureStrat} \ixParam^t$.
      The result is that}
    \utility^{1:T}(\pureStrat) - \est{\utility}^{1:T}(\pureStrat)
      &\le \dfrac{\maxReturn}{2 \ixParam^{\textsc{min}}} \log\subex*{\dfrac{1}{\delta'}}
    \shortintertext{
      holds with probability $1 - \delta'$ for any $\delta' \in (0, 1)$.
      Choosing $\delta' = \frac{\delta}{\abs{\PureStrategySet} + 1}$ yields
    }
    \utility^{1:T}(\pureStrat) - \est{\utility}^{1:T}(\pureStrat)
      &\le \dfrac{\maxReturn}{2 \ixParam^{\textsc{min}}} \log\subex*{\dfrac{\abs{\PureStrategySet} + 1}{\delta}}.
  \end{align}

  To summarize again, our bound is now
  \begin{align}
    \regret^{1:T}(\pureStrat)
      &\le
        g(T, \abs{\PureStrategySet}, \delta, \tuple{\ixParam^t}_{t = 1}^T)
        + \maxReturn \abs{\PureStrategySet} \sum_{t = 1}^T
          \ixParam^t
        + \dfrac{\maxReturn}{2 \ixParam^{\textsc{min}}} \log\subex*{\dfrac{\abs{\PureStrategySet} + 1}{\delta}}
        + \sum_{t = 1}^T
          \ixParam^t
            \sum_{\pureStrat' \in \PureStrategySet}
              \utility(\pureStrat'; \daimonStrat^t) - \est{\utility}^t(\pureStrat').
  \end{align}

  The last term to simplify is
  $\sum_{t = 1}^T
    \ixParam^t
      \sum_{\pureStrat' \in \PureStrategySet}
        \utility(\pureStrat'; \daimonStrat^t) - \est{\utility}^t(\pureStrat')$
  and we largely repeated our last simplification procedure.
  The rounds where
  $\CixCorrection(\PureStrat^t; \strat^t, \ixParam^t) = 1$
  can again be ignored because
  $\utility(\pureStrat'; \daimonStrat^t) - \est{\utility}^t(\pureStrat') \le 0$
  on those rounds, so
  \begin{align}
    \sum_{t = 1}^T
      \ixParam^t
        \sum_{\pureStrat' \in \PureStrategySet}
          \utility(\pureStrat'; \daimonStrat^t) - \est{\utility}^t(\pureStrat')
      &\le
        \sum_{t \in \uncappedRounds}
          \ixParam^t
            \sum_{\pureStrat' \in \PureStrategySet}
              \utility(\pureStrat'; \daimonStrat^t) - \est{\utility}^t(\pureStrat')\\
      &=
        \dfrac{\maxReturn}{2} \sum_{t \in \uncappedRounds}
          \sum_{\pureStrat' \in \PureStrategySet}
            \dfrac{2\ixParam^t}{\maxReturn}
              \subex*{
                \utility(\pureStrat'; \daimonStrat^t) - \est{\utility}^t(\pureStrat')
              }.
    \shortintertext{
      Applying \cref{lem:cramerChernoff} again with
      $\lambda^t_{\pureStrat'} = \frac{\ixParam}{\strat^t(\pureStrat')}$
      and
      $\alpha^t_{\pureStrat'} = 2 \ixParam^t$,
    }
    \sum_{t = 1}^T
      \ixParam^t
        \sum_{\pureStrat' \in \PureStrategySet}
          \utility(\pureStrat'; \daimonStrat^t) - \est{\utility}^t(\pureStrat')
      &\le
        \dfrac{\maxReturn}{2} \log\subex*{
          \dfrac{\abs{\PureStrategySet} + 1}{\delta}
        }.
  \end{align}

  Therefore, the final parameterized bound is
  \begin{align}
    \regret^{1:T}(\pureStrat)
      &\le
        g(T, \abs{\PureStrategySet}, \delta, \tuple{\ixParam^t}_{t = 1}^T)
        + \maxReturn \abs{\PureStrategySet} \sum_{t = 1}^T
          \ixParam^t
        + \dfrac{\maxReturn}{2 \ixParam^{\textsc{min}}} \log\subex*{\dfrac{\abs{\PureStrategySet} + 1}{\delta}}
        + \dfrac{\maxReturn}{2} \log\subex*{
          \dfrac{\abs{\PureStrategySet} + 1}{\delta}
        },
  \end{align}
  as required.

  The no-regret bound is realized by using the integral trick to bound
  \begin{align}
    \sum_{t = 1}^T
      \xi
        \sqrt{
          \frac{1}{\abs{\PureStrategySet} t}
        }
      \le
        2 \xi
          \sqrt{
            \frac{T}{\abs{\PureStrategySet}}
          }
  \end{align}
  and by recognizing that
  $\ixParam^{\textsc{min}} = \xi
    \sqrt{
      \frac{1}{\abs{\PureStrategySet} T}
    }$.
\end{proof}

\end{document}